\documentclass{hld2025} 

\title[Exploration Behavior of Untrained Policies]{Exploration Behavior of Untrained Policies}
\usepackage{graphicx}
\usepackage{caption}
\usepackage{wrapfig}

\newtheorem{assumption}{Assumption}

\usepackage{float}

\hldauthor{%
\Name{Jacob Adamczyk} \Email{jacob.adamczyk001@umb.edu}\\
\addr {Department of Physics, University of Massachusetts Boston, United States \\ The NSF AI Institute for Artificial Intelligence and Fundamental Interactions} 
}

\begin{document}
\maketitle

\begin{abstract}
Exploration remains a fundamental challenge in reinforcement learning (RL), particularly in environments with sparse or adversarial reward structures. In this work, we study how the architecture of deep neural policies implicitly shapes exploration before training. We theoretically and empirically demonstrate strategies for generating ballistic or diffusive trajectories from untrained policies in a toy model. Using the theory of infinite-width networks and a continuous-time limit, we show that untrained policies return correlated actions and result in non-trivial state-visitation distributions. We discuss the distributions of the corresponding trajectories for a standard architecture, revealing insights into inductive biases for tackling exploration. Our results establish a theoretical and experimental framework for using policy initialization as a design tool to understand exploration behavior in early training.
\end{abstract}

\section{Introduction}
Effective exploration is crucial for reinforcement learning agents operating in high-dimensional or sparse-reward environments. While much focus has been placed on designing explicit exploration bonuses or strategies, we shift attention to a more implicit source of exploration: the parameterization of the policy network. Importantly, studying the effect of policy architecture does not require training of additional networks~\cite{burda2018exploration}, or introduction of exploration-dependent rewards~\cite{taiga2021bonus, lobel2023flipping}. 
Untrained policies are responsible for determining the initial data distribution from which deep RL agents begin training. As such, the properties of deep policies at initialization is important for understanding exploration in the early training regime. We hypothesize that the choice of architecture and the corresponding initialization distribution implicitly determine the entropy and geometry of exploration. As a step toward understanding this phenomenon, we provide some initial theoretical and experimental results to support this hypothesis. We leverage infinite-width and continuous-time limits to understand random policy behavior in a toy model with corresponding simulations. 

\textbf{Contributions:} In this work, we show that (1) fixed policies lead to ballistic trajectories where correlations and smoothness in the network dominate (2) random policy initializations at each step can be used to generate diffusive trajectories with a heavy-tailed steady-state distribution and (3) combining these methods can provide new pathways for controlled exploration in deep RL. 

\textbf{Related Work:} Exploration has been studied from many angles~\cite{ladosz2022exploration}. Our work leverages the infinite-width limit~\cite{neal2012bayesian, jacot2018neural} and analysis of the Fokker-Planck equation which have both proven to be fruitful avenues of research for various ML communities. Neural architecture search (NAS) might attempt to optimize the policy architecture directly~\cite{zoph2016neural}, e.g. for better exploration strategies. Rather than searching in the space of possible architectures, we instead approach the problem through the lens of closed-form solutions, attempting to understand the inductive biases of a fixed architecture and its impact on exploration behavior.

\section{Background}
\subsection{Reinforcement Learning}
Focusing on the effects of untrained policies, we only need the basic and usual definitions: a state space, $\mathcal{S}$, action space $\mathcal{A}$, deterministic transition dynamics (a function mapping state and action to successor-state) $f:\mathcal{S}\times\mathcal{A}\to\mathcal{S}$. Let $\pi_\theta : \mathcal{S} \to \mathcal{A}$ be a policy represented by a feedforward neural network with parameters $\theta\in\mathbb{R}^n$ randomly initialized according to some specified scheme (e.g. Xavier-Glorot). We will not introduce any (intrinsic or exploration-dependent) reward functions and instead focus on the reward-free Markov processes that govern the agent's data collection process. 

Operating in the policy-based RL setting, a network $\pi_\theta$ is trained to maximize returns (e.g. REINFORCE~\cite{williams1992simple}, TRPO~\cite{schulman2015trust}, PPO~\cite{schulman2017proximal}). Value-based algorithms such as $Q$-learning and its offspring~\cite{watkins1992q, mnih2015human, hessel2018rainbow} typically use $\epsilon$-greedy approaches, and the additional non-linear mapping from value to policy space can complicate the analysis, so this connection is left to future work.

\subsection{Smooth Networks Yield Ballistic Trajectories~\label{sec:ballistic}}
For short times, sufficiently smooth neural networks can induce trajectories with a dominating drift term: ``ballistic motion''. Loosely speaking, the structure of a neural net induces similar outputs (actions) for nearby states, and when compounded with smooth dynamics, this results in agents with very smooth (non-diffuse) trajectories. This result is visualized in Figures~\ref{fig:minipage1} and~\ref{fig:minipage2} and is formalized more precisely below:
\begin{assumption}[Transition Dynamics Locality]
\label{assump:dyn}
The environment dynamics are ``local'': there exists a $\delta>0$ such that for all $s\in \mathcal{S}$ and $a\in\mathcal{A}$, the distance between $s$ and any corresponding successor state $s'=f(s,a)$ is upper-bounded $|s'-s|<\delta$.
\end{assumption}

\begin{lemma}[Short-Time Ballistic Behavior of Lipschitz Neural Policies]
Let $\pi_\theta: \mathbb{R}^d \to \mathbb{R}^d$ be a stochastic policy sampled from a neural network prior, such that the realized sample function is $L_\pi$-Lipschitz. Let the deterministic transition dynamics $s_{t+1}=f(s_t,a_t)$ satisfy Assumption 1, and also be Lipschitz continuous with respect to the action space: $|f(s,a)-f(s,a')| \leq L_a |a-a'|$ and state space, $|f(s,a)-f(s',a)|\leq L_s|s-s'|$ for all $s,s'\in\mathcal{S}$ and  $a,a'\in\mathcal{A}$.

Let the constant $c=\pi_\theta(s_0)$ denote the initial action and let $L_s=1$ (the $L_s \neq 1$ regime is dominated by exponential time dependence, and is further discussed in the Appendix). Then for any initial state $s_0\in\mathcal{S}$, the agent's trajectory approximately follows the straight-line path $ct$, up to a bounded error:
\[
    |(s_t - s_0)- ct| \leq \frac{1}{2}\delta L_a L_\pi t^2.
\]
\end{lemma}

\begin{remark}
    Although the previous lemma holds for all times, a linear approximation to the trajectory: ${s_t=s_0 + ct +\mathcal{O}(t^2)}$ holds only for short timescales, $t \ll \frac{2c}{\delta L_a L_\pi}$, beyond which, deviations caused by curvature (resulting from non-linearities in the deep neural net) become large enough to dominate the effective mean velocity,~$c$.
\end{remark}

The Lipschitz constant of a neural policy can be computed~\cite{virmaux2018lipschitz, latorre2020lipschitz, bhowmick2021lipbab} or roughly bounded {\textit{a~priori}} to make Lemma~1 practically implementable. Although the assumptions and structure of the above dynamics may be limiting for some environments, the result provides a starting point to develop an intuition for the relationship between architecture, dynamics, and ``depth''\footnote{In the sense of ``deep exploration''~\cite{NIPS2016_8d8818c8}.} of trajectories. We will focus on generalizing these results in future work. 

If a fixed policy network is used at the beginning of training (e.g. to fill up a replay buffer), the agent will not observe diverse trajectories when periodically reset to the initial state, (a common setup in episodic RL). As an extreme alternative, we can instead consider initializing a \textit{new} policy (from a fixed distribution, say) at each timestep and study its effects on the agent's trajectories.

\begin{figure}[H]
    \centering
    \begin{minipage}[t]{0.4\textwidth}
        \centering
        \includegraphics[width=\textwidth]{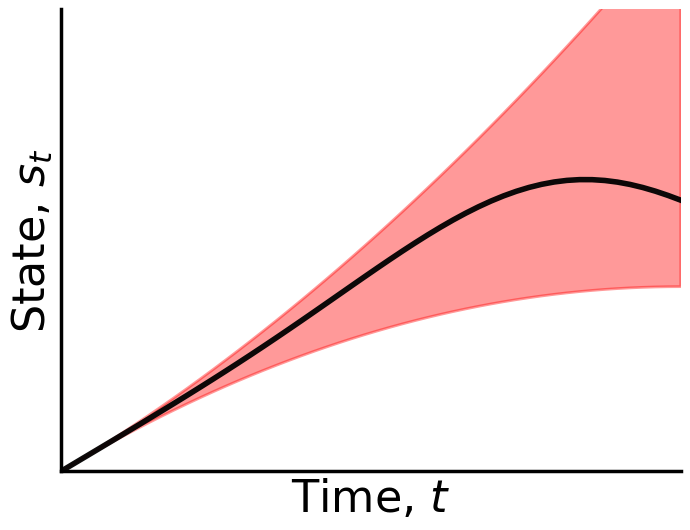}
        \captionof{figure}{Though trajectories can change direction (as observed in the plot on the right), on short timescales, the trajectories can be well-approximated with a dominant linear drift.}
        \label{fig:minipage1}
    \end{minipage}
    \hspace{5em}
    \begin{minipage}[t]{0.3\textwidth}
        \centering
        \includegraphics[width=\textwidth]{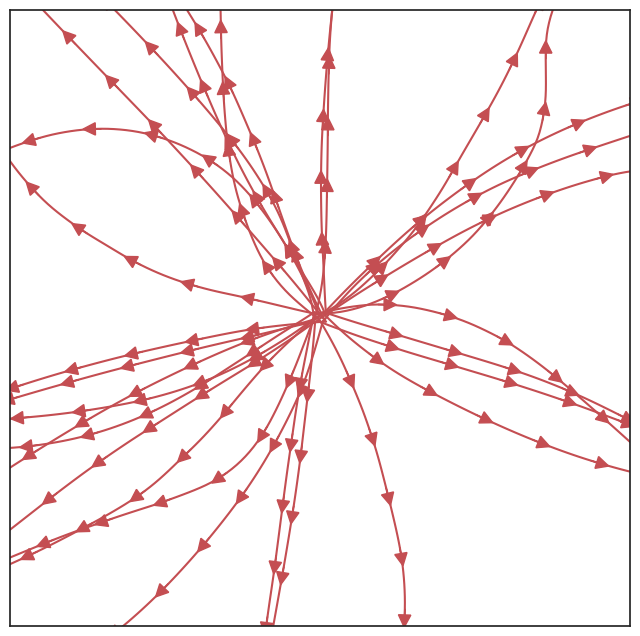}
        \captionof{figure}{Standard MLPs used for deep RL (with two hidden layers, $256$ hidden nodes, ReLU activation) produce ``ballistic'' trajectories.}
        \label{fig:minipage2}
    \end{minipage}
\end{figure}

\subsection{Random Policy Networks as Gaussian Processes}
Rather than a single policy generating an entire trajectory (as is typically the case), we seek to obtain a more diffusive (and hence potentially more exploratory) trajectory distribution. To this end, the control policy $\pi_\theta$ is now re-sampled sequentially, producing a stochastic trajectory ensemble (even with deterministic dynamics and deterministic action outputs from the policies):

\begin{definition}[Trajectory Under a Random Policy Ensemble]
Let $\mathcal{T}_\theta = \{s_0, s_1, \dots, s_T\}$ denote a trajectory generated by:
\[
s_{t+1} = f\left(s_t, \pi_{\theta_t}(s_t)\right), \quad \theta_t \sim \mathcal{P}_\theta
\]
where each $\theta_t$ is independently drawn from the initialization distribution $\mathcal{P}_\theta$, and $\pi_{\theta_t}$ is the corresponding deterministic policy network.
\end{definition}

With this setup, we can leverage the following closed-form expression for the policy in the infinite-width limit (cf. also Eq. 1 of \cite{meronen2021periodic} and surrounding discussion therein for more details):
\begin{theorem}[\citet{neal2012bayesian}, Infinite-Width Limit as Gaussian Process]
Let $\pi_\theta$ be a feedforward neural network with nonlinear activation $\phi(0)=0$ (e.g., ReLU or Tanh), and i.i.d. weights and biases, with zero mean. In the limit as the width of all hidden layers tends to infinity, the distribution over functions $\pi_\theta$ converges to a Gaussian Process:
\[
\pi_\theta \overset{d}{\longrightarrow} \mathcal{GP}(0, K(s, s')),
\]
where $K$ is an architecture-dependent kernel.
\end{theorem}
The GP then produces a structured action covariance when a new policy is sampled at each timestep. Formally, let states $s, s' \in \mathcal{S}\doteq\mathbb{R}^d$ be given. The covariance between actions (w.r.t. draws from the initial parameter distribution) at the two states is given by: 
${\text{Cov}\left[\pi_\theta(s), \pi_\theta(s')\right] = K(s, s') \in \mathbb{R}^{d \times d}}$. Thus, the kernel corresponding to the chosen architecture (and hence the parameter initialization) gives rise to action correlations across the state-space. Below, we show an analytically tractable example of this structure in a policy with one (infinitely wide) hidden layer and ReLU activation, for its simplicity and popularity in the literature.

Note that these properties are complementary to that of the fixed policy, whose architecture is typically Lipschitz, implying the ballistic trajectories discussed in Section~\ref{sec:ballistic}.

\subsection{Case Study: ReLU Network and the Induced State Distribution}

For ReLU networks with Gaussian weight initialization, the infinite-width kernel is given by:
\begin{equation}
K(s,s') =  \frac{\sigma_w^2}{\pi} |s||s'| \left(\sin \theta + (\pi - \theta) \cos \theta\right) + \sigma_b^2,
\end{equation}
where $\cos \theta = \frac{\langle s, s' \rangle}{|s||s'|}$. This kernel is not stationary (i.e. it does not only depend on a radial distance between inputs), and induces a diffusion term that grows with $|s|^2$. Thus, far from the origin, action selection becomes less correlated. Stationary kernels (e.g. Radial Basis Functions~\cite{buhmann2000radial}) depending on $|s-s'|$ on the other hand would result in a constant diffusion coefficient, which may be of interest in some environments.

In the special case of one hidden layer, infinitely wide, with ReLU non-linearities, we take weights and biases drawn from centered Gaussians with variance $\sigma_w, \sigma_b$, respectively. In this case, the policy is mean zero and has diffusion coefficient $K(s,s)=\Sigma(s) = \sigma_b^2 + \frac{\sigma_w^2}{\pi} \|s\|^2 $. For the simplest linear dynamics $s_{t+1} = s_t + \pi(s_t)$, we can model the stochastic policy as a random walk in state space. In the continuous time (equivalently ``small action'') limit, the Fokker-Planck equation can be employed to study the dynamical distribution $p(s,t)$ over state space:
\begin{equation}
\frac{\partial p}{\partial t} = -\mu_0^\top \nabla p + \frac{1}{2} \nabla^2 : (\Sigma(s) p),
\end{equation}
\noindent which in the long-time limit (assuming stationarity) can be written as:
\[
\nabla^2 \left[ \left( \sigma_b^2 + \frac{\sigma_w^2}{\pi} \|s\|^2 \right) p_\infty(s) \right] = 0.
\]

\noindent With radial symmetry $p_\infty(s) = f(r)$, $r = \|s\|$, we obtain the following solution:

\begin{equation}
f(r) \propto\left( \sigma_b^2 + \frac{\sigma_w^2\cdot r^2}{\pi} \right)^{-d/2} \sim{} \frac{1}{r^d}.
\end{equation}

This describes a heavy-tailed stationary distribution (resembling Cauchy or power-law distributions) and is normalizable if $\sigma_b>0$. Note that as $\sigma_b^2 \to 0$, the tails become increasingly heavy, and the agent is more likely to reach distant regions of state space.

\subsection{Experiments}
Unsurprisingly, an agent's exploration behavior depends on the nature of its policy parameterization. To demonstrate our theoretical results, we show that ballistic or diffusive motion can be controlled via policy ``resets''. As illustrated in Figure~\ref{fig:exploration-inline}, placing a barrier in state space with a narrow hallway drastically reduces the likelihood of uniformly exploring state space (i.e., passing through the hallway). Ballistic trajectories, generated by a fixed MLP (red), often do not pass through the hallway: they follow straight paths, and small errors result in failed exploration. In contrast, the trajectories from policy networks that are re-initialized at every step (green) exhibit diffusive,\begin{wrapfigure}{r}{0.48\textwidth}
    \centering
    \includegraphics[width=\linewidth]{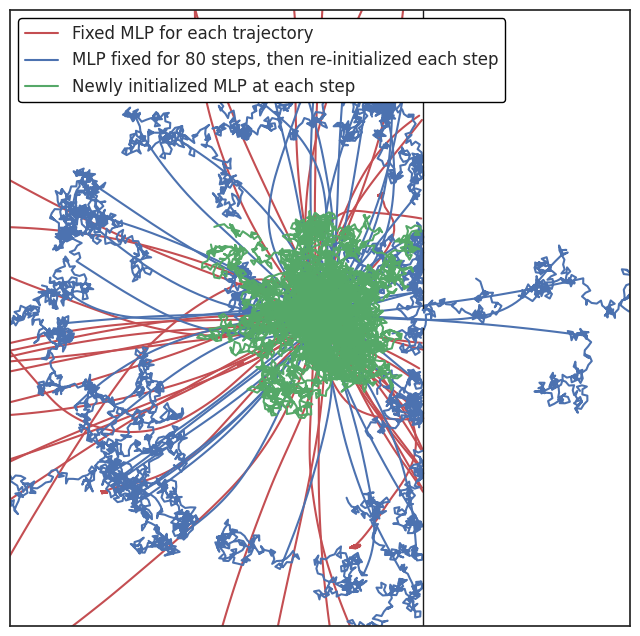}
    \caption{\textbf{Exploration through a narrow hallway.} Ballistic trajectories from a fixed MLP struggle to pass through the barrier; stepwise re-initialization produces overly-diffusive motion, while a hybrid strategy leverages both behaviors for efficient exploration.}
    \label{fig:exploration-inline}
\end{wrapfigure} fat-tailed distributions that will eventually (but slowly) explore across the barrier. A hybrid switching strategy, where a fixed MLP is used for the first $n$ steps before switching to per-step re-initialization, captures both the initial linear drift and later diffusion, enabling more effective exploration in this scenario. Within the linear regime, choosing a value of ${n\ll L_\pi^{-1}}$ (cf. Lemma 1) ensures that trajectories will escape the initial state, while $n > 2\Delta/\delta$ can ensure that agents reach the states of interest (e.g. those a distance $\Delta$ from the origin) before diffusing. Studying this tradeoff with a more quantitative definition of ``exploration'' in more complex environments is the focus of future work. The policy reset can also be done stochastically with increasing probability, based on agent observations, which may be especially relevant in the deep RL setting, perhaps providing additional insight to the ~\cite{NEURIPS2023_75101364}.

\section{Discussion}

In this work, we explored how the architecture and initialization of policy networks can shape the exploration dynamics of reinforcement learning agents. By modeling untrained policies as draws from Gaussian Processes (GPs) in the infinite-width limit, we demonstrated that distinct neural architectures and parameterizations implicitly induce nontrivial priors over agent behavior. For simple dynamics models, these behaviors are analytically tractable, giving insight into the exploration distribution. These priors manifest themselves as specific patterns in trajectory geometry, ranging from ballistic drift to heavy-tailed diffusive exploration.

Our analysis showed that Lipschitz continuity in deterministic policy networks leads to smooth, directional (ballistic) trajectories which may limit early exploration. In contrast, policies sampled at each timestep from an architectural prior induce highly stochastic (but structured) trajectory distributions. We studied these trajectories in the continuous-time limit, employing a Fokker–Planck equation, revealing a connection between a network's kernel and steady-state distribution. We found that ReLU-based policies generate quasi-Cauchy state distributions in free space, encouraging broad exploration. Notably, the magnitude and shape of the induced kernel dictate the correlation between actions at different states, determining whether trajectories are locally consistent or rapidly de-correlating, when using the per-step policy sampling scheme.

These findings suggest a new lens for understanding exploration. Rather than a process to be tuned through learning or incentivized via external bonuses, one can envisage exploration as an architectural and initialization design problem. By tailoring policy architectures to match the topology or reward sparsity of a target environment, one can influence early-stage exploration ``zero-shot''. Future work will study the relationship to resetting parameters during training~\cite{nikishin2022primacy}, which may have the additional benefit of improving exploration ability. This discussion aligns with a broader incentive to align network architectures with environments, an area in deep reinforcement learning that has not yet been explored.

\subsection*{Acknowledgements}
JA acknowledges funding support from the PDT Partners Machine Learning Conference Grant and the NSF through Award No. PHY-2425180.
This work is supported by the National Science Foundation under Cooperative Agreement PHY-2019786 (The NSF AI Institute for Artificial Intelligence and Fundamental Interactions, http://iaifi.org/).

\bibliography{main}
\newpage

\section*{Appendix}
We first provide a useful inequality for obtaining our main result, Lemma 1, first focusing on the case of $L_s=1$:
\begin{lemma}
The sequence of inequalities
\begin{equation}
    \epsilon_{t+1} \leq \epsilon_t + kt
\end{equation}
has iterates that are upper bounded by $\epsilon_t \leq kt^2/2 +\epsilon_0$.\label{lemma:helper}
\end{lemma}
\begin{proof}
The proof is straightforward by induction. For the base case, note that $\epsilon_0\leq k(0)^2/2 + \epsilon_0$. From the inductive hypothesis, 
\begin{align*}
\epsilon_{t+1} 
&\leq \epsilon_t + kt \\
&\leq \frac{kt^2}{2} + \epsilon_0 + kt \\
&= \frac{k(t^2 + 2t)}{2} + \epsilon_0 \\
&= \frac{k(t+1)^2 - k}{2} + \epsilon_0 \\
&= \frac{k(t+1)^2}{2} + \epsilon_0 - \frac{k}{2} \\
&\leq \frac{k(t+1)^2}{2} + \epsilon_0
\end{align*}
 which is the form of iterate $(t+1)$, thus completing the proof.
\end{proof}

When $L_s\neq 1$, the recursive form is instead controlled by an exponential behavior, $L_s^t$:
\begin{lemma}
The sequence of inequalities
\begin{equation}
    \epsilon_{t+1} \leq A\epsilon_t + kt
\end{equation}
has iterates that are upper bounded by $\epsilon_t \leq kt\frac{A^t-1}{A-1} +\epsilon_0$.\label{lemma:helper2}
\end{lemma}
\begin{proof}
    The proof is similar to that of Lemma~\ref{lemma:helper}, except that a factor of $A$ is accumulated at each step in a geometric series.
\end{proof}

\noindent We now prove Lemma~1:

\begin{proof}
    First, denote $\widetilde{s}_t=ct+s_0$, the linear approximation of the trajectory. We examine the iterates of the error, $\epsilon_t\doteq |s_t-\widetilde{s}_t|$ to verify the bound presented in the lemma.

    \begin{align}
        \epsilon_{t+1}&=|s_{t+1}-\widetilde{s}_{t+1}| \\
        &= |f(s_{t}, \pi_\theta(s_t))-f(\widetilde{s}_{t}, c)| \\
        &= |f(s_{t}, \pi_\theta(s_t))-f(s_t,c) + f(s_t,c)-f(\widetilde{s}_{t}, c)|\\
        &\leq |f(s_{t}, \pi_\theta(s_t))-f(s_t,c)|  + |f(s_t,c)-f(\widetilde{s}_{t}, c)|\\
        &\leq L_a|\pi_\theta(s_t)-\pi_\theta(s_0)|  + L_s|s_t-\widetilde{s}_{t}|\\
        &\leq L_a L_\pi|s_t - s_0|  + L_s \epsilon_{t}\\
        &\leq L_a L_\pi\delta t  + L_s \epsilon_{t}
    \end{align}
    Several notes are in order: In the last three lines we (a) made the substitution of $c=\pi_\theta(s_0)$ as an estimate of the mean velocity, (b) used the Lipschitz property of the neural network $\pi_\theta$, and (c) used the locality of dynamics iteratively to bound the distance in state-space. For (c), we have iterated the bound in Assumption~1 for multiple timesteps, making use of the triangle inequality.
    
    Note that for (a), this choice leads naturally to the desired result and is easy to compute (once the first action is drawn, the estimate of the drift is complete). However, it may not lead to the tightest bounds in practice. Empirically, we have found that using a mean of actions along the trajectory can smooth out the estimate for velocity, somewhat improving the bound. Similar steps can be taken to arrive at a linear upper bound, using that e.g. $|\pi_\theta(s_t)-\langle \pi_\theta(s_t) \rangle|<g\varepsilon$. where $g$ is a geometric factor and $\varepsilon$ is the diameter of a ball over which the average $\langle \pi_\theta(s_t) \rangle$ is computed. In any case, the resulting bound is linear.

    Using Lemma~\ref{lemma:helper}, we see that $\epsilon_t\leq \delta L_a L_\pi t^2/2$, neglecting the initial error term $\epsilon_0=0$ (since at $t=0$ we have the correct state $\widetilde{s}_0=s_0$ and there is no approximation). Ensuring the error does not accumulate past the linear prediction ($\delta L_a L_\pi t^2/2<ct$) results in the upper bound on timesteps presented in the main text.

    When considering $L_s\neq 1$, Lemma~\ref{lemma:helper2} applies and there are two distinct cases: (1) If $L_s<1$ the dynamics are contractive, and trajectories remain roughly linear ($\propto t(1-e^{-t/\tau})$) beyond a timescale $\tau=\ln L_s$. Intuitively, this should be expected since a Lipschitz constant less than unity is subsumed by the weaker case of $L_s=1$. (2) If $L_s>1$ the dynamics are explosive and trajectories exponentially separate ($\propto t(e^{t/\tau}-1)$). Although a large global Lipschitz constant $L_s>1$ may be required from a theoretical standpoint, in practical settings, much of state-space may be governed by relatively small local Lipschitz constants $L_s \lesssim 1$, making (sub)trajectories more well-behaved than na\"ively expected by this worst-case analysis.
\end{proof}

\end{document}